\documentclass[journal]{IEEEtran}
\usepackage{amsmath}
\usepackage{amssymb}
\usepackage{latexsym}
\usepackage{psfrag}
\usepackage{graphicx,subfig}
\usepackage{epstopdf}
\usepackage{multicol}
\usepackage{fleqn}
\usepackage{enumerate}
\usepackage{color}
\usepackage[colorlinks,linkcolor=blue,citecolor=blue]{hyperref}
\usepackage{mathrsfs}
\usepackage{algorithm}
\usepackage{algorithmic}
\usepackage{xcolor}
\usepackage{cite}
\newcommand{\tabref}[1]{TABLE~\ref{#1}}
\usepackage{type1cm}
\usepackage{booktabs}
\usepackage{float}
\usepackage[marginal]{footmisc}
\usepackage{array}
\usepackage{cases}

\newenvironment{proof}[1][Proof]{%
	\par\noindent\hspace{1em}\textit{#1.} \ignorespaces
}{%
	\par\noindent\hfill$\square$\par\vspace{0.5em}
}

\newtheorem{theorem}{Theorem}
\newtheorem{definition}{Definition}

\newtheorem{lemma}{Lemma}
\newtheorem{remark}{Remark}
\allowdisplaybreaks[4]

\usepackage{hyperref}
\usepackage{booktabs}

\begin{document}

\title{Safe Reinforcement Learning Filter for Multicopter Collision-Free Tracking under disturbances}
\author{Qihan Qi, Xinsong Yang, Gang Xia,
\thanks{This work was supported in part by the National Natural Science Foundation of China (NSFC) under Grant Nos. 62373262 and 62303336, and in part by the Central guiding local science and technology development special project of Sichuan, and in part by the Fundamental Research Funds for Central Universities under Grant No. 2022SCU12009, and in part by the Sichuan Province Natural Science Foundation of China (NSFSC) under Grant Nos. 2022NSFSC0541,  2022NSFSC0875, 2023NSFSC1433, and in part by the National Funded Postdoctoral Researcher Program of China under Grant GZB20230467, in part by the China Postdoctoral Science Foundation under Grant 2023M742457, in part by the Power Quality Engineering Research Center of Ministry of Education under Grant KFKT202305.  Corresponding authors: Xinsong Yang.}\\
\thanks{Q.~Qi, X.~Yang, and G. Xia are with the College of Electronics and Information Engineering, Sichuan University, Chengdu 610065, China (e-mails: qiqihan@stu.scu.edu.cn (Q. Qi); xinsongyang@163.com or xinsongyang@scu.edu.cn (X. Yang); 17623110370@163.com (X. Gang)).}

}
\maketitle
\begin{abstract}
This paper proposes a safe reinforcement learning filter (SRLF) to realize multicopter collision-free trajectory tracking with input disturbance. A novel robust control barrier function (RCBF) with its analysis techniques is introduced to avoid collisions with unknown disturbances during tracking. To ensure the system state remains within the safe set, the RCBF gain is designed in control action. A safety filter is introduced to transform unsafe reinforcement learning (RL) control inputs into safe ones, allowing RL training to proceed without explicitly considering safety constraints. The SRLF obtains rigorous guaranteed safe control action by solving a quadratic programming (QP) problem that incorporates forward invariance of RCBF and input saturation constraints. Both simulation and real-world experiments on multicopters demonstrate the effectiveness and excellent performance of SRLF in achieving collision-free tracking under input disturbances and saturation.
\end{abstract}

\begin{IEEEkeywords}
Multicopter control, safe reinforcement learning, control barrier function, disturbances.
\end{IEEEkeywords}

\section{Introduction}
Multicopters exhibit exceptional six degrees of freedom motion capabilities due to their unique design, enabling precise translational and rotational movements in three-dimensional space. Their distinctive dynamic characteristics allow for complex maneuvers, including omnidirectional linear displacements and rapid attitude changes \cite{hamandi2020survey}. Consequently, multicopters hold significant potential for applications such as line inspection \cite{li2024multi, li2023uav}, search and rescue \cite{wu2023adaptive, moro2024enhancing}, and precision agriculture \cite{gokool2023crop, lachgar2023unmanned}. The issue of trajectory tracking is prevalent across these applications, prompting the development of various control techniques aimed at achieving effective trajectory tracking \cite{maaruf2022survey}, including linear quadratic regulator (LQR) control, model predictive control (MPC), fuzzy control, and reinforcement learning (RL) control.

RL has demonstrated remarkable achievements in the field of multicopters control \cite{ song2023reaching, messikommer2024contrastive}, even surpassing human multicopter racing champions in real-world \cite{kaufmann2023champion}. The remarkable achievements in RL for multicopter control have attracted significant attention from researchers. However, challenges in RL multicopter control such as  real-world disturbances, and the lack of safety guaranteed action control impede the transition of RL from simulations to real-world applications. To address the safety issues in RL, a variety of approaches have been proposed, which can be broadly divided into two types\cite{zhao2023state}: safe learning methods\cite{tessler2018reward,yang2023safety} and safety filter methods\cite{wabersich2021predictive, chen2024learning, dalal2018safe,zhao2021model}. The safe learning methods  optimize the RL policy with safe constraints throughout the learning process directly,  which cannot be rigorously guaranteed safety since RL policy have to balance maximizing reward and minimizing safety cost. Additionally, the learned safety from safe learning methods may not be applicable in real-world tracking due to the differences between simulation and real-world, especially facing real-world input disturbance.

In contrast, safety filter methods address the safety problem by transforming potentially unsafe RL controls into safe control outputs through a safety filter, ensuring that safety can be rigorously guaranteed.  Researchers have developed various approaches to construct safety filters. \cite{wabersich2021predictive, chen2024learning} designed a predictive-based filter to achieve safety through MPC, but the long predictive horizon may require high computational costs. \cite{dalal2018safe} proposed a safety filter by constructing a safety signal correlated with the system state. However, this approach relies on an assumption of a linear relationship between states and control outputs. Furthermore, \cite{dalal2018safe} does not account for the potential safety implications of current control outputs on future system states. Additionally, these methods can hardly handle the disturbances which is the focus of this paper, let alone address disturbances and input saturation at the same time.

In the field of safe RL, the method of implementing safety filter through constructing  control barrier function (CBF) has attracted significant attention \cite{ames2016control, cheng2019end, hu2023safe}.  This approach can generate safe control outputs to ensure the forward invariance of the safe set without paying the computational costs associated with model prediction \cite{wabersich2021predictive, chen2024learning}. Notably, CBF-based safety filter enables the integration of arbitrary all model-free RL algorithms and can rigorously guarantee safe by filtering out unsafe inputs. However, existing CBF-based safe RL methods \cite{ames2016control, cheng2019end, hu2023safe} either do not consider or hard to solve the impact of input disturbance, due to the lack of corresponding analysis techniques, resulting in the inability to guarantee the forward invariance of the safe set in both simulation and real-world.  These gaps motivate us to develop novel analysis techniques for addressing input disturbance in CBF, with the aim of achieving collision avoidance trajectory tracking for multicopters.

This paper proposes a SRLF method to realize multicopter collision-free trajectory tracking with input disturbance. A novel RCBF with its analysis techniques is proposed to avoid collisions with unknown disturbances during tracking. The RCBF gain is designed to ensure the system state remains within the safe set despite input disturbances. On the basis of RCBF and its gain, a safety filter is designed to transform unsafe RL control to safe one. The SRLF obtains rigorous guaranteed safe control action by solving the QP problem with forward invariance of RCBF and input saturation constraints.  The main contributions are as follows.
\begin{enumerate}
\item[{(1)}] We propose the collision avoidance RCBF for multicopter trajectory tracking, the RCBF gain is designed in control action to guarantee the safe set forward invariance under input disturbance.
\item[{(2)}] A safety filter is proposed that transforms potentially unsafe control inputs from any model-free RL algorithm into collision-free safe controls for multicopter tracking. This filter allows RL training to proceed without explicitly considering safety constraints, simplifying the overall training process.
\item[{(3)}] The SRLF framework and two training methods are proposed, using a QP problem with RCBF constraints to ensure collision-free tracking under input disturbances and saturation. Simulations and real-world multicopter experiments demonstrate the effectiveness and excellent performance of SRLF.
\end{enumerate}

The rest of the paper is organized as follows. Section II presents some preliminaries. Section III introduced a  RCBF and RCBF gain for control action to ensure safe set forward invariance. Section IV explores the dynamics of multicopter and develops a specific  RCBF for safe trajectory tracking. In Section V, the SRLF framework is constructed. The simulation and real-world deployment results are provided in  Section VI. Finally, Section VII gives the conclusion of this paper.

\begin{table}[htbp]\centering\caption{Notations}
	\begin{tabular}{ll}
		\toprule
		Notation & Description \\
		\midrule
		$P^T$ & Transpose of vector or matrix $P$ \\
		$P^{-1}$ & Inverse of vector or matrix $P$ \\
		$\partial P$ &  Boundary of a closed set $P$ \\
		$\text{Int}(P)$ &  Interior of a closed set $P$ \\
		$\mathbb{R}^n$ & Set of $n$-dimensional Euclidean vector space \\
		$\mathbb{R}^{n\times m}$ & Set of $n\times m$ real matrix space \\
		$\|\cdot\|$ & Euclidean norm \\
		$\|\cdot\|_{\infty}$ & Infinity norm \\
		$SO(3)$ & Special Orthogonal Group in three dimensions \\
		\bottomrule
		\end{tabular}
\end{table}

\section{Preliminaries}
\subsection{Model-free RL}
Model-free RL problems can be modeled as a Markov decision process  $(\mathcal{X}, \mathcal{U}, r, \mathcal{P}, \zeta)$, where $\mathcal{X}$ and $\mathcal{U}$ are the continuous state space and continuous control action space, respectively. The reward function is denoted by $r: \mathcal{X} \times \mathcal{U}  \rightarrow \mathbb{R}$, and the state transition function is represented by $\mathcal{P}: \mathcal{X} \times \mathcal{U} \times \mathcal{X} \rightarrow \left[ 0, 1 \right]$, $\zeta$ is a discount factor.
It is assumed that the state $\mathbf{x}_{t}\in \mathcal{X}$ at time $t$ can be observed from the environment, the agent takes a control action $\mathbf{u}_{t}\in \mathcal{U}$ to interact with the environment and transition from state $\mathbf{x}_{t}$ to $\mathbf{x}_{t+1}$.  $\pi(\cdot|\mathbf{x}_{t})$ is the control action policy distribution under state $\mathbf{x}_{t}$ and control action $\mathbf{u}_{t}\sim\pi(\cdot|\mathbf{x}_{t})$.  The model-free RL optimization problem can be represented as
\begin{align} \label{2.1}
	\max_{\pi}\mathop{\mathbb{E}}\limits_{} \bigg[\sum\limits_{t=0}^{\infty}\zeta^{t}r_{t}\bigg].
\end{align}

For on-policy model-free RL, such as Proximal Policy Optimization (PPO) \cite{schulman2017proximal}, the actor objective can be expressed as:
\begin{align}
	J_{\mathbf{u}_{r}}(\theta_{\mathbf{u}_{r}})=\mathbb{E}[\min(\rho(\theta_{\mathbf{u}_{r}})\hat{A}, \text{clip}(\rho(\theta_{\mathbf{u}_{r}}), 1 - \hat{\epsilon}, 1 + \hat{\epsilon})\hat{A}_t)],
\end{align}
where $\theta_{\mathbf{u}_{r}}$ denotes the actor network parameters, $\rho(\theta_{\mathbf{u}_{r}})$ is the probability ratio between the new and old policy, $\hat{A}_t$ represents the estimated advantage, and $\hat{\epsilon}$ is a hyperparameter for clipping.

The critic objective for PPO is typically defined as
\begin{align}
	J_{c}(\theta_{c}) = \mathbb{E}\left[\left( r_{t} + \zeta V_{\theta_{c}}(\mathbf{x}_{t+1}) -  V_{\theta_{c}}(\mathbf{x}_{t})\right)^2\right],
\end{align}
where $V_{\theta_{c}}(\cdot): \mathcal{X} \rightarrow \mathbb{R}$ is the critic network with parameters $\theta_{c}$ for value estimation .

For the off-policy model-free RL, such as Soft Actor-Critic (SAC) \cite{haarnoja2018soft}, the actor objective is formulated as

\begin{align}
	J_{\mathbf{u}_{r}}(\theta_{\mathbf{u}_{r}}) = \mathbb{E}\left[\log \pi(\mathbf{u}_t | \mathbf{x}_t) - Q_{\theta_{c}}(\mathbf{x}_t, \mathbf{u}_t)\right],
\end{align}
where \( Q_{\theta_{c}}(\mathbf{x}_t, \mathbf{u}_t) \) is critic network with parameters $\theta_{c}$ for estimating  Q-value.

The critic objective for SAC can be expressed as
\begin{align}
	J_{c}(\theta_{c}) = \mathbb{E}\left[\left(Q_{\theta_{c}}(\mathbf{x}_t, \mathbf{u}_t) - \hat{y}_t\right)^2\right],
\end{align}
where $ \hat{y}_t = r_t + \zeta Q_{\theta_{c}}(\mathbf{x}_{t+1}) $ is the target for the Q-function.


\subsection{CBF}
For notational simplicity, we represent $\mathbf{x}_{t}$ as $\mathbf{x}$. Consider the following affine control system:
\begin{align} \label{2.2}
\dot{\mathbf{x}}=f(\mathbf{x})+g(\mathbf{x})\mathbf{u}
\end{align}
where $\mathbf{x}\in \mathbb{R}^{n}$ is the system state,  $\mathbf{u}\in \mathbb{U} \subset \mathbb{R}^{m}$ denotes control action,  $\mathbb{U}$ represents control action set, $f(\cdot): \mathbb{R}^{n} \rightarrow \mathbb{R}^{n}$,  $g(\cdot): \mathbb{R}^{n} \rightarrow \mathbb{R}^{n\times m}$ are locally Lipschitz continuous functions.

The following definition is given on the basis of (\ref{2.2}).
\begin{definition} \label{def1}
For system \eqref{2.2},  given a safe set $\mathcal{C} \subset \mathcal{D} \subseteq \mathbb{R}^{n}$ with a smooth function $b(\cdot): \mathbb{R}^{n}\rightarrow\mathbb{R}$, the function  $b(\cdot)$ is a CBF defined on the set $\mathcal{D}$,  if the following condition holds
\begin{align}\label{2.7}
	\sup_{\mathbf{u} \in \mathbb{U}} [\mathcal{L}_f b(\mathbf{x}) + \mathcal{L}_g b(\mathbf{x}) \mathbf{u}] \geq -\alpha(b(\mathbf{x})),
\end{align}
where $\mathcal{L}_f b(\mathbf{x})$ and $\mathcal{L}_g b(\mathbf{x})$ are the Lie derivatives, $\alpha(\cdot): \mathbb{R} \rightarrow \mathbb{R}$ is a given class $\mathcal{K}$ function.
\end{definition}

Note that, given a smooth function $b(\cdot)$, the safe set $\mathcal{C}$ can be represented by
\begin{align}
 	\mathcal{C}=\{\mathbf{x} \in \mathbb{R}^{n} : b(\mathbf{x})  \geq 0\}. \label{2.3}
 \end{align}
It is shown in \cite{ames2016control} that $\mathcal{C}$ is forward invariant for any acceptable control action $\mathbf{u}$.

The following notations will be used in later study.
 \begin{align}
 	\partial\mathcal{C} &= \{\mathbf{x} \in \mathbb{R}^{n} : b(\mathbf{x}) = 0\}, \label{2.4} \\
 	\text{Int}(\mathcal{C}_d) &= \{\mathbf{x} \in \mathbb{R}^{n} : b(\mathbf{x}) > 0\}, \label{2.5}\\
 	\mathcal{D} &= \{\mathbf{x} \in \mathbb{R}^{n} : b(\mathbf{x}) + \bar{\iota}  > 0\}, \label{2.6}
 \end{align}
where $\bar{\iota}$ is a positive constant.

\section{Robust Control Barrier Function}
In real environment, unknown disturbances are unavoidable. the ideal system \eqref{2.2} is impractical. This section proposes a RCBF to ensure the safety of a control system with unknown disturbances.

Define a closed set by accounting for the model disturbances $\boldsymbol{\mu}$.
\begin{align}
	\mathcal{C}_d &= \{\mathbf{x} \in \mathbb{R}^{n} : b(\mathbf{x}) + \iota(\|\boldsymbol{\mu}\|_\infty) \geq 0\}, \label{3.4} \\
	\partial\mathcal{C}_d &= \{\mathbf{x} \in \mathbb{R}^{n} : b(\mathbf{x}) + \iota(\|\boldsymbol{\mu}\|_\infty) = 0\}, \label{3.5} \\
	\text{Int}(\mathcal{C}_d) &= \{\mathbf{x} \in \mathbb{R}^{n} : b(\mathbf{x}) + \iota(\|\boldsymbol{\mu}\|_\infty) > 0\}, \label{3.6}
\end{align}
where $\iota(\cdot) \in \mathcal{K}_{[0,\hat{d})}$, $\|\boldsymbol{\mu}\|_\infty \leq \bar{d} \in [0, \hat{d})$, $\bar{d}>0$ denotes a constant, $\hat{d}$ is a constant slightly larger than $\bar{d}$. The function $\iota(\cdot)$ meets $\mathop{\lim}\limits_{a\rightarrow \hat{d}} \iota(a)=\bar{\iota}$, therefore we can ensure that $\iota(\bar{d}) < \bar{\iota}$, which implies that $\mathcal{C}_{d}\subset\mathcal{D}$.

\begin{definition}
Let $b(\cdot): \mathbb{R}^{n} \to \mathbb{R}$ be a continuously differentiable function defined on the two sets $\mathcal{C}$ and $\mathcal{D}$. If there exists a set of controls $\mathbb{U}$, two functions $\alpha(\cdot) \in \mathcal{K}_{(-\bar{\iota},c)}$,  $\beta(\cdot) \in \mathcal{K}_{[0,\hat{d})}$ satisfying $\mathop{\lim}\limits_{a\rightarrow \hat{d}} \beta(a)=\bar{\iota}$, and a constant $\bar{d} \in [0, \hat{d})$ such that, for $\forall \mathbf{x} \in \mathcal{D}$, $\forall \mathbf{d} \in \mathbb{R}^{n}$ satisfying $\|\mathbf{d} \|\leq\bar{d}$, there holds
\begin{align} \label{3.7}
\sup_{\mathbf{u} \in \mathbb{U}} [\mathcal{L}_f b(x) + \mathcal{L}_g b(x)(\mathbf{u} + \boldsymbol{\mu})] \geq -\alpha(b(\mathbf{x})) - \beta(\|\mathbf{d} \|),
\end{align}
then the function $b(\cdot)$ is called a RCBF defined on $\mathcal{D}$.

\end{definition}

\begin{lemma}[ \cite{kolathaya2018input}] \label{lemma1}
	Consider the system \eqref{2.2} with disturbance $\boldsymbol{\mu}$,  letting $b(\cdot): \mathbb{R}^{n}\rightarrow\mathbb{R}$ be a smooth function defined on $\mathcal{D}$. If $b(\cdot)$ is a RCBF with $\mathcal{C}\subset\mathcal{C}_{d}\subset\mathcal{D}$, and the set $\mathcal{C}_{d}$ satisfies $\iota(\cdot) \in \mathcal{K}_{[0,\hat{d})}$, $\|\boldsymbol{\mu}\|_\infty \leq \bar{d} \in [0, \hat{d})$, $\bar{d}>0$,  $\mathop{\lim}\limits_{a\rightarrow \hat{d}} \iota(a)=\bar{\iota}$, then the sets $\mathcal{C}_{d}$ and $\mathcal{C}$ are safe sets under disturbances.
\end{lemma}

It should be noted that in many tasks, there are several RCBFs, e.g. collisions avoidance. Hence, the following approximation of RCBFs is designed
\begin{align}
	b(\mathbf{x}) = -\frac{1}{\varrho} \ln \left( \sum_{i=1}^M \exp(-\varrho b_{i}(\mathbf{x})) \right)
\end{align}
where $\varrho$ is a given positive constant, $b_{i}(\mathbf{x})$ represents the $i$th RCBF,  $M$ denotes the number of RCBFs, $i=1,\cdots,M$.

Similar with \eqref{2.3}, the safe set $\mathcal{C}_{i}$ can be denoted as
\begin{align}
	\mathcal{C}_{i}=\{\mathbf{x} \in \mathbb{R}^{n} : b_{i}(\mathbf{x})  \geq 0\}.
\end{align}

\begin{remark} \label{remark1}
	It can be obtained from \cite{li2024quadrotor} that $b(\mathbf{x}) = -\frac{1}{\varrho}\ln \left( \sum_{i=1}^M \exp(-\varrho b_{i}(\mathbf{x})) \right) \leq \min_{i\in\{1,\cdots,M\}} (b_{i}(\mathbf{x}))$. Then, $b(\mathbf{x}) > 0$ can ensure that $b_{i}(\mathbf{x})\geq0$, and $\mathcal{C}\subseteq\bigcap\limits_{i=1}^{M} \mathcal{C}_{i}\subset\mathcal{D}$ is the safe set, $i=1,\cdots,M$.
\end{remark}

Inspired by \cite{kolathaya2018input,sontag1989smooth}, the controller can be designed as follows to make sure that the set $\mathcal{C}$ is safe and forward invariant,
\begin{align}\label{3.8}
	\mathbf{u}=\mathbf{u}_{s}+\mathcal{L}_{g}b(\mathbf{x})^{T},
\end{align}
where $\mathbf{u}_{s}\in\mathbb{R}^{3}$ is the control from safety filter, $\mathcal{L}_{g}b(\mathbf{x})^{T}=\frac{\sum_{i=1}^M \left( \exp(-\varrho b_{i}(\mathbf{x})) \mathcal{L}_{g}b_{i}(\mathbf{x})^{T} \right) }{\sum_{i=1}^M \exp(-\varrho b_{i}(\mathbf{x}))}$ is the RCBF gain for control output.

\begin{theorem} \label{theorem1}
	Consider a series of continuously differentiable functions $b_{i}(\cdot): \mathbb{R}^{n} \to \mathbb{R}$ defined on $\mathcal{C}_{i}$, $i=1,\cdots,M$, an open set $\mathcal{D}$ and a set of controls $\mathbb{U}$, if the approximated $b(\cdot)= -\frac{1}{\varrho} \ln \left( \sum_{i=1}^M \exp(-\varrho b_{i}(\cdot)) \right)$ satisfies
	\begin{align}\label{3.9}
		\sup_{\mathbf{u} \in \mathbb{U}} [\mathcal{L}_f b(\mathbf{x}) + \mathcal{L}_g b(\mathbf{x}) \mathbf{u} - \mathcal{L}_g b(\mathbf{x}) \mathcal{L}_g b(\mathbf{x})^T] \geq -\alpha(b(\mathbf{x})),
	\end{align}
	for some $\alpha(\cdot) \in \mathcal{K}_{(-\bar{\iota},c)}$, and for all $\mathbf{x} \in \mathcal{D}$, then $b(\cdot)$ and $b_{i}(\cdot)$ are RCBFs defined on the set $\mathcal{D}$.
\end{theorem}

\begin{proof}
	After substituting \eqref{3.9} into the derivative of $b(\cdot)$:
	\begin{align*}
		\dot{b}(\mathbf{x}) &= \sup_{\mathbf{u} \in \mathbb{U}} [\mathcal{L}_f b(\mathbf{x}) + \mathcal{L}_g b(\mathbf{x})(\mathbf{u} + \boldsymbol{\mu})] \\
		&\geq -\alpha(b(\mathbf{x})) + \mathcal{L}_g b(\mathbf{x}) \mathcal{L}_g b(\mathbf{x})^T + \mathcal{L}_g b(\mathbf{x}) \boldsymbol{\mu} \\
		&\geq -\alpha(b(\mathbf{x})) + \|\mathcal{L}_g b(\mathbf{x})\|^2 - \|\mathcal{L}_g b(\mathbf{x})\| \|\boldsymbol{\mu}\|_\infty,
	\end{align*}
	since $\mathcal{L}_g b(\mathbf{x}) \mathcal{L}_g b(\mathbf{x})^T = \|\mathcal{L}_g b(\mathbf{x})\|^2$. Adding and subtracting $\beta(\|\boldsymbol{\mu}\|)=\frac{1}{4} \|\boldsymbol{\mu}\|_\infty^2$ yields
	\begin{align*}
		\dot{b}(\mathbf{x}) &\geq -\alpha(b(\mathbf{x})) + \left(\|\mathcal{L}_g b(\mathbf{x})\| - \frac{\|\boldsymbol{\mu}\|_\infty}{2}\right)^2 - \frac{\|\boldsymbol{\mu}\|_\infty^2}{4} \\
		&\geq -\alpha(b(\mathbf{x})) - \frac{\|\boldsymbol{\mu}\|_\infty^2}{4},
	\end{align*}
	which is of the form \eqref{3.7}.  Therefore,  $b(\mathbf{x})$ is a RCBF on the set $\mathcal{D}$.  It can be obtained from  Remark \ref{remark1} that if $b(\mathbf{x})$ is a RCBF, it ensures that each $b_{i}(\mathbf{x})$ is also a RCBF, and $\mathcal{C}\subseteq\bigcap\limits_{i=1}^{M} \mathcal{C}_{i}\subset\mathcal{D}$. Subsequently, in accordance with Lemma \ref{lemma1}, we can derived that $\mathcal{C}_{d}$ is a safe set under disturbance. The proof is completed.
\end{proof}

\begin{remark}
Given a CBF $b(\mathbf{x})$, we propose a slightly enlarged set $\mathcal{C}_{d}$ that maintains forward invariance. The RCBF ensures that the system state, even when subjected to input disturbances, remains confined within or in close proximity to a smaller safe set $\mathcal{C}$. The degree of this proximity is denoted by the magnitude of the disturbance $\bar{\iota}$.
\end{remark}

\begin{remark}
Introducing RCBF gains into the control action $\mathbf{u}$ can ensure that forward invariance is maintained under disturbances. Compared to \cite{ames2016control, cheng2019end, hu2023safe}, forward invariance can be challenging to satisfy due to disturbances. In particular, these control algorithms \cite{ames2016control, cheng2019end, hu2023safe} are optimization-based, which aim to maximize or minimize objectives to obtain the best performance, but causing the system state to approach the boundary $\partial \mathcal{C}$, which may lead to exiting the safe set $\mathcal{C}$ under disturbances.
\end{remark}

\section{Multicopter Dynamics and Robust Control Barrier Function Formulation}
In this section, we explore multicopter dynamics and develop a RCBF approach for safe trajectory tracking. We present the fundamental equations of multicopter motion and introduce a simplified model incorporating disturbances. We then propose a collision avoidance strategy using RCBFs for multiple obstacles, including safety constraints and barrier functions. Our goal is to develop a SRLF controller ensuring collision-free tracking under disturbances and input saturation, compatible with model-free reinforcement learning algorithms.

Consider the following multicopter system equations:
\begin{align} \notag
\dot{\mathbf{p}} &= \mathbf{v}, \quad \dot{\mathbf{v}} = \frac{f}{m} \mathbf{R}\mathbf{e}_{3}-g\mathbf{e}_{3}, \\
\dot{\mathbf{a}} &= \mathbf{w}, \quad \dot{\mathbf{w}} = \mathbf{J}^{-1}(-\mathbf{w}\times \mathbf{J}\mathbf{w}+\boldsymbol{\tau}), \label{3.1}
\end{align}
where $\mathbf{p}=[p_{x},p_{y},p_{z}]^{T}$ and $\mathbf{v}=[v_{x},v_{y},v_{z}]^{T}\in\mathbb{R}^{3}$ are the position and velocity of the multicopter in the world frame, respectively. $\mathbf{a}=[a_{\phi},a_{\theta},a_{\psi}]^{T}\in\mathbb{R}^{3}$ represents the roll, pitch, and yaw angles, and $\mathbf{w}=[\omega_{x},\omega_{y},\omega_{z}]^{T}\in\mathbb{R}^{3}$ denotes the angular velocity in the body frame. $\mathbf{R}\in SO(3)$ is the rotation matrix from the body frame to the world frame, $\mathbf{e}_{3}=[0,0,1]^{T}$, $g$ is the gravity acceleration, $m$ denotes mass of multicopter, $f$ is the total thrust. $\mathbf{J}\in\mathbb{R}^{3\times3}$ is the inertia matrix, and $\boldsymbol{\tau}\in\mathbb{R}^{3}$ is the control torque.

Let $\mathbf{x}=[\mathbf{p}^{T}, \mathbf{v}^{T}]^{T}$, according to differential flatness of the multicopter system\cite{mellinger2011minimum}, the system state $[\mathbf{p}^{T}, \mathbf{v}^{T}, \mathbf{a}^{T}, \mathbf{w}^{T}]^{T}$ in \eqref{3.1} and control input $f, \boldsymbol{\tau}$ can be written as $[\mathbf{p}^{T}, a_{\psi}]^{T}$ and their derivatives, then the multicopter system with disturbances can be represented as

\begin{align}\label{3.2}
&\dot{\mathbf{x}}=\left[
\begin{array}{c}
\dot{\mathbf{p}}\\
\dot{\mathbf{v}}
\end{array}
\right] =\underbrace{\left[
\begin{array}{c@{\hspace{0.3em}}c}
	0,&\mathbf{I}_{3}\\
	0,&0
\end{array}
\right]
\left[
\begin{array}{c}
	\mathbf{p}\\
	\mathbf{v}
\end{array}
\right]}_{f(\mathbf{x})}+\underbrace{\left[
\begin{array}{c}
	0\\
	\mathbf{I}_{3}
\end{array}
\right]}_{g(\mathbf{x})}(\mathbf{u}+\boldsymbol{\mu}),  \\
&\dot{a}_\psi =w_{c}, \label{3.21}
\end{align}
where $\mathbf{I}_{3}\in\mathbb{R}^{3\times 3}$ is identity matrix,  $\mathbf{u}\in\mathbb{R}^{3}$ and $w_{c}\in\mathbb{R}$ are the control commands for velocity and yaw rate, respectively, which can be resolved into corresponding rotor speeds by low-level controllers such as PX4 Autopilot, $\boldsymbol{\mu}\in\mathbb{R}^{3}$ represents the disturbances, which is caused by the body drag and parasitic aerodynamic drag forces arising from the multicopter's flight and thrust generation.

Next, we proposed the collision avoidance RCBF with its RCBF gain for multicopter trajectory tracking. Consider the following safety constraint  with $N$ obstacles
\begin{align}\label{3.10}
	\|\Delta\mathbf{p}_{i}\| + \int_{t_0}^{t_0 + T_f} \Delta\mathbf{v}_{i}(t_0 + t) \, dt \geq D_s, i=1,\cdots,N,
\end{align}
where $\Delta\mathbf{p}_{i}=\mathbf{p}_{i}-\mathbf{p}$, $\mathbf{p}_{i}$ is the position of $i$th obstacle, $\Delta\mathbf{v}_{i}(t_0)=\mathbf{v}_{i}(t_{0})-\mathbf{v}(t_{0})$, $t_{0}$ denotes the current time, $\Delta\mathbf{v}_{i}(t_0+t)=\Delta\mathbf{v}_{i}(t_{0})+\delta t$,  $\delta$  represents maximum braking acceleration of multicopter, $T_{f}=\frac{0-\Delta\mathbf{v}_{i}(t_0)}{\delta}$. $D_{s}$ is safety distance.

Intuitively, we only need to focus on the projection of $\Delta\mathbf{v}_{i}$ onto $\Delta\mathbf{p}_{i}$ for collision avoidance. On the basis of this observation, $b_{i}(\mathbf{x})$, $b(\mathbf{x})$  and safe set $\mathcal{C}$ can be obtained from \eqref{3.10}:
\begin{equation}\label{3.11}
	\begin{aligned}
		b_{i}(\mathbf{x})=\sqrt{2\delta(\|\Delta\mathbf{p}_{i}\|-D_{s})}
		+\frac{\Delta\mathbf{p}_{i}^{T}}{\|\Delta\mathbf{p}_{i}\|}\Delta\mathbf{v}_{i}, \\
		b(\mathbf{x})=-\frac{1}{\varrho} \ln \left( \sum_{i=1}^M \exp(-\varrho b_{i}(\mathbf{x})) \right)\\
		\mathcal{C}=\{x\in\mathbb{R}^{n}|b_{i}(\mathbf{x})\geq0\}, i=1,\cdots,N.
	\end{aligned}
\end{equation}

Considering unknown disturbances, the control action is designed with an RCBF gain $\mathcal{L}_{g}b_{i}(\mathbf{x})=\frac{\Delta\mathbf{p}_{i}^{T}}{\|\Delta\mathbf{p}_i\|}$ to ensure safe set $\mathcal{C}$ forward invariance. By combining \eqref{3.8} and \eqref{3.9}, and letting
$\alpha(b(\mathbf{x}))=\gamma b^{3}(\mathbf{x})$, $\bar{e}=\sum_{i=1}^M \exp(-\varrho b_{i}(\mathbf{x}))$, $\hat{e}=\sum_{i=1}^{M}(\exp(-\varrho b_{i}(\mathbf{x}))\cdot( \frac{\delta\Delta\mathbf{v}_i^T\Delta\mathbf{p}_i}{\sqrt{2\delta(\|\Delta\mathbf{p}_i\| -  D_s)}\|\Delta\mathbf{p}_{i}\|}
+ \frac{\|\Delta\mathbf{v}_{i}\|^{2}}{\|\Delta\mathbf{p}_i\|}
-\frac{(\Delta\mathbf{v}_i^T\Delta\mathbf{p}_i)^2}{\|\Delta\mathbf{p}_i\|^{3}}+\frac{\Delta\mathbf{p}_{i}^T\Delta\mathbf{p}_{i}}{\|\Delta\mathbf{p}_i\|^{2}}))$,
$\tilde{e}_{i}=\exp(-\varrho b_{i}(\mathbf{x}))\cdot\frac{\Delta\mathbf{p}_{i}^{T}}{\|\Delta\mathbf{p}_i\|}$, $i=1,\cdots,M$, $\tilde{e}=[\tilde{e}_{1},\cdots,\tilde{e}_{M}]^{T}$,
$\tilde{\mathbf{u}}_{s}=[\underbrace{\mathbf{u}_{s}, \mathbf{u}_{s}, \ldots, \mathbf{u}_{s}}_{M \text{ times}}]$, the forward invariance safety barrier constraint can be expressed as
\begin{align}\label{3.12}
	-\tilde{e}^{T}\tilde{\mathbf{u}}_{s}\leq\hat{e}+\gamma b^{3}(\mathbf{x})
\end{align}

To achieve trajectory tracking, the reference state can be represented as $\mathbf{x}_{r}=[\mathbf{p}_{r}^{T},\mathbf{v}_{r}^{T}]^{T}$, and $\mathbf{p}_{r}$, $\mathbf{v}_{r}\in\mathbb{R}^{3}$ are the position and velocity of the reference trajectory, $a_{\psi,r}$ is the reference yaw angle. Suppose there exists a control action $\mathbf{u}^{*}$ that minimizes the tracking error $\|\mathbf{x}_{r}-\mathbf{x}\|$ while satisfying collision avoidance. We aim to train an SRLF controller $\mathbf{u}=\mathbf{u}_{s}+\mathcal{L}_{g}b(\mathbf{x})^{T}$ such that $\|\mathbf{u}^{*}-\mathbf{u}\| \rightarrow 0$, ensuring strictly guaranteed collision-free trajectory tracking under unknown disturbances and input saturation. This approach is compatible with all model-free RL algorithms to achieve trajectory tracking.

\section{Safe Reinforcement Learning Filter Framework}
In this section, a SRLF is presented for multicopter trajectory tracking with collision avoidance and input saturation. The framework of SRLF integrates a model-free RL algorithm for trajectory tracking with a practical RCBF and input constraints based filter to ensure collision avoidance. The approach is formulated as a QP  problem that considers input saturation and unknown disturbances. A theorem is presented and proved, demonstrating that the SRLF ensures collision-free tracking under unknown disturbances and input saturation.

\begin{figure}[!htbp]
	\centering
	\hspace{-1cm}
	\includegraphics[scale=0.25]{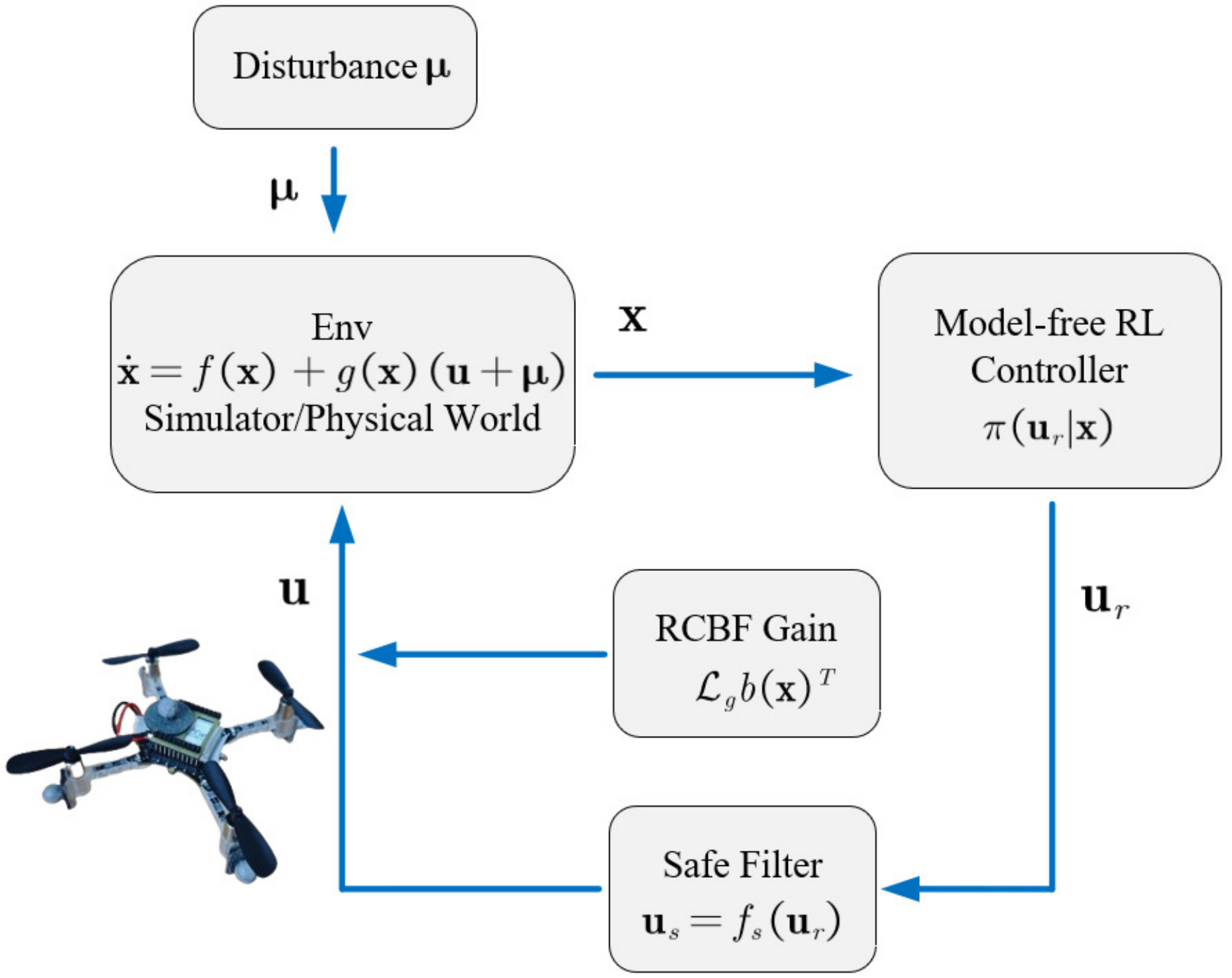}
	\caption{The framework of SRLF.} \label{fig1}
\end{figure}

It should be noted that RL performs excellently on many tasks. However, the lack of safety constraints severely limits the application of RL, let alone real-world deployment. Additionally, though some safety learning methods such as PPO-Lag \cite{yu2022towards} can achieve soft safety, they increase the burden of task learning since policy should consider both rewards and safety costs, which tend training to failed, especially for complex collision-free multicopter trajectory tracking. Therefore, we  propose a SRLF that help RL focus on  multicopter trajectory tracking and filter the RL control action to meet collision-free based on RCBF.

The framework of SRLF is shown in Fig. \ref{fig1}, where $\mathbf{u}_{s}$ is the control from SRLF, $\mathbf{u}_{r}$ is the control that can be obtained by arbitrary model-free RL algorithm, which is presented in Algorithm \ref{alg1}. $f_{s}$ is the filter function and $\mathbf{u}_{s}=f_{s}(\mathbf{u}_{r})=\mathbf{u}_{r}+\mathbf{u}_{f}$, $\mathbf{u}_{f}$ denotes the compensatory control input generated by the filter to ensure safety constraints are satisfied.

The control output of SRLF $\mathbf{u}=\mathbf{u}_{r}+\mathbf{u}_{f}+\mathcal{L}_{g}b(\mathbf{x})$, it can be obtained form \eqref{3.11} that  $\mathcal{L}_{g}b(\mathbf{x})=\frac{\sum_{i=1}^M \left( \exp(-\varrho b_{i}(\mathbf{x})) \mathcal{L}_{g}b_{i}(\mathbf{x})^{T} \right) }{\sum_{i=1}^M \exp(-\varrho b_{i}(\mathbf{x}))}=\frac{1}{\bar{e}}\sum_{i=1}^M \exp(-\varrho b_{i}(\mathbf{x}))\frac{\Delta \mathbf{p}_{i}^{T}}{\|\Delta \mathbf{p}_{i}\|}$. Then, let $\tilde{\mathbf{u}}_{r}=[\underbrace{\mathbf{u}_{r}, \mathbf{u}_{r}, \ldots, \mathbf{u}_{r}}_{M \text{ times}}]$ and $\tilde{\mathbf{u}}_{f}=[\underbrace{\mathbf{u}_{f}, \mathbf{u}_{f}, \ldots, \mathbf{u}_{f}}_{M \text{ times}}]$, \eqref{3.12} can be rewritten as
\begin{align*}
-\tilde{e}^{T}\tilde{\mathbf{u}}_{f}\leq\hat{e}+\gamma b^{3}(\mathbf{x})+\tilde{e}^{T}\tilde{\mathbf{u}}_{r}
\end{align*}

To achieve SRLF,  the optimal filter control $\mathbf{u}_{f}^{*}$ can be solved by the following QP with input saturation and unknown disturbances:
\begin{equation}\label{4.1}
\begin{aligned}
\mathbf{u}_{f}^{*} & =\arg\min_{\mathbf{u}_{f}} \|\mathbf{u}_{f}\| \\
\text{s.t.} \quad & \mathbf{u}_{\min} \leq \mathbf{u}_{r} + \mathbf{u}_{f} + \mathcal{L}_{g}b(\mathbf{x}) \leq \mathbf{u}_{\max}, \\
&-\tilde{e}^{T}\tilde{\mathbf{u}}_{f}\leq\hat{e}+\gamma b^{3}(\mathbf{x})+\tilde{e}^{T}\tilde{\mathbf{u}}_{r}.
\end{aligned}
\end{equation}

\begin{algorithm}[!htbp]
	\caption{Model-free RL Algorithm Training}
\label{alg1}
\begin{algorithmic}[1]
	\REQUIRE Initialized actor and critic network parameters $\theta_{\mathbf{u}_{r}}$, $\theta_{c}$, learning rate $\eta_{\mathbf{u}_{r}}$,  $\eta_{c}$, loss $J_{c}(\theta_{c})$, $J_{\mathbf{u}_{r}}(\theta_{\mathbf{u}_{r}})$, total training steps $\mathbf{M}$, update interval $\mathbf{n}$.
	\ENSURE Unsafe actor network parameters $\theta_{\mathbf{u}_{r}}$.
	\STATE Set current training step $\mathbf{m}=0$
	\WHILE {$\mathbf{m}<\mathbf{M}$}
	\STATE Observe state $\mathbf{x}_{t}$
	\STATE Select and execute action $\mathbf{u}_{r,t}\sim\pi_{\theta_{\mathbf{u}_{r}}}(\cdot|\mathbf{x}_{t})$
	\STATE Observe reward $r_{t}$ and  next state $\mathbf{x}_{t+1}$
	\STATE Store and get tuple $(x_{t},u_{t},r_{t},x_{t+1})$ in Buffer
	\IF{$\mathbf{m} \mod \mathbf{n} == 0$}
	\STATE Update critic network parameters $\theta_{c}\leftarrow \theta_{c}-\eta_{c}\nabla_{\theta_{c}} J_{c}(\theta_{c})$
	\STATE Update actor network parameters $\theta_{\mathbf{u}_{r}}\leftarrow\theta_{\mathbf{u}_{r}}+\eta_{\mathbf{u}_{r}}\nabla_{\theta_{\mathbf{u}_{r}}}J_{\mathbf{u}_{r}}(\theta_{\mathbf{u}_{r}})$
	\ENDIF
	\STATE$\mathbf{m}=\mathbf{m}+1$
	\ENDWHILE
\end{algorithmic}
\end{algorithm}

\begin{theorem}
Assume that the initial tracking error of multicopter and the disturbances are bounded. If the controller is successfully trained using an arbitrary model-free RL algorithm such that $\|\mathbf{u}^{*}-\mathbf{u}\| \rightarrow 0$, the QP problem with forward invariance of RCBF  and input saturation constraints is feasible at any time $t$, then the collision-free tracking with disturbances and input saturation can be guaranteed.
\end{theorem}

\begin{proof}
Since the control input $\mathbf{u}_{r}$ is successfully trained, it implies that $\mathbf{u}_{r}$ can solve basic tracking problem. The theorem \ref{theorem1} proves that  the control action in the form of \eqref{3.8} can keep the set $\mathcal{C}$ safe and forward invariant, and the constraint $ \mathbf{u}_{\min} \leq \mathbf{u}_{r} + \mathbf{u}_{f} + \sum_{i=1}^{N}\frac{\Delta\mathbf{p}_i}{\|\Delta\mathbf{p}_i\|} \leq \mathbf{u}_{\max}$,  $-\Delta\mathbf{p}_{i}^T\mathbf{u}_{f} \leq h_{1} + h_{2}, i=1,\cdots,N$ in \eqref{4.1} makes sure collision-free and input saturation, therefore the control action $\mathbf{u}$ can guarantee the collision-free tracking of multicopter with input disturbances and saturation. The proof is completed.
\end{proof}

\section{Experiments}
In this section, the multicopter figure-8 tracking experiments are designed in numerical simulation, and a  Crazyflie 2.1 quadcopter is utilized to carry out the figure-8 tracking with disturbances and input saturation in real-world experiment to verify the effectiveness and safety of the SRLF.
In the experiments, we have implemented the SRLF using two distinct approaches:
\begin{enumerate}
	\item \text{Safe Post Filtering:} This method involves training a model-free Reinforcement Learning (RL) agent without considering collision avoidance constraints. Subsequently, we apply a safety filter to the control actions generated by the agent to ensure collision-free operation. This approach is characterized by its relative ease of training and higher success rate in achieving convergence.
	\item \text{Safe Learning Filtering:} In this approach, we integrate the safety filter into both the training and deployment phases. While this method ensures safe training procedures, it typically requires a longer training duration. The advantage lies in its comprehensive safety considerations throughout the entire learning and execution process.
\end{enumerate}
The first method prioritizes training efficiency, while the second emphasizes safety throughout the entire learning process. Both approaches aim to achieve collision-free tracking, but they differ in their integration of safety constraints within the SRLF. The choice between these methods depends on the specific requirements of the application, balancing the trade-off between training efficiency and comprehensive safety assurance. The following experiments are conducted on a computer with the i7-13700K CPU and RTX 4060ti GPU.

\begin{figure}[!h]
\centering
\includegraphics[width=0.5\textwidth]{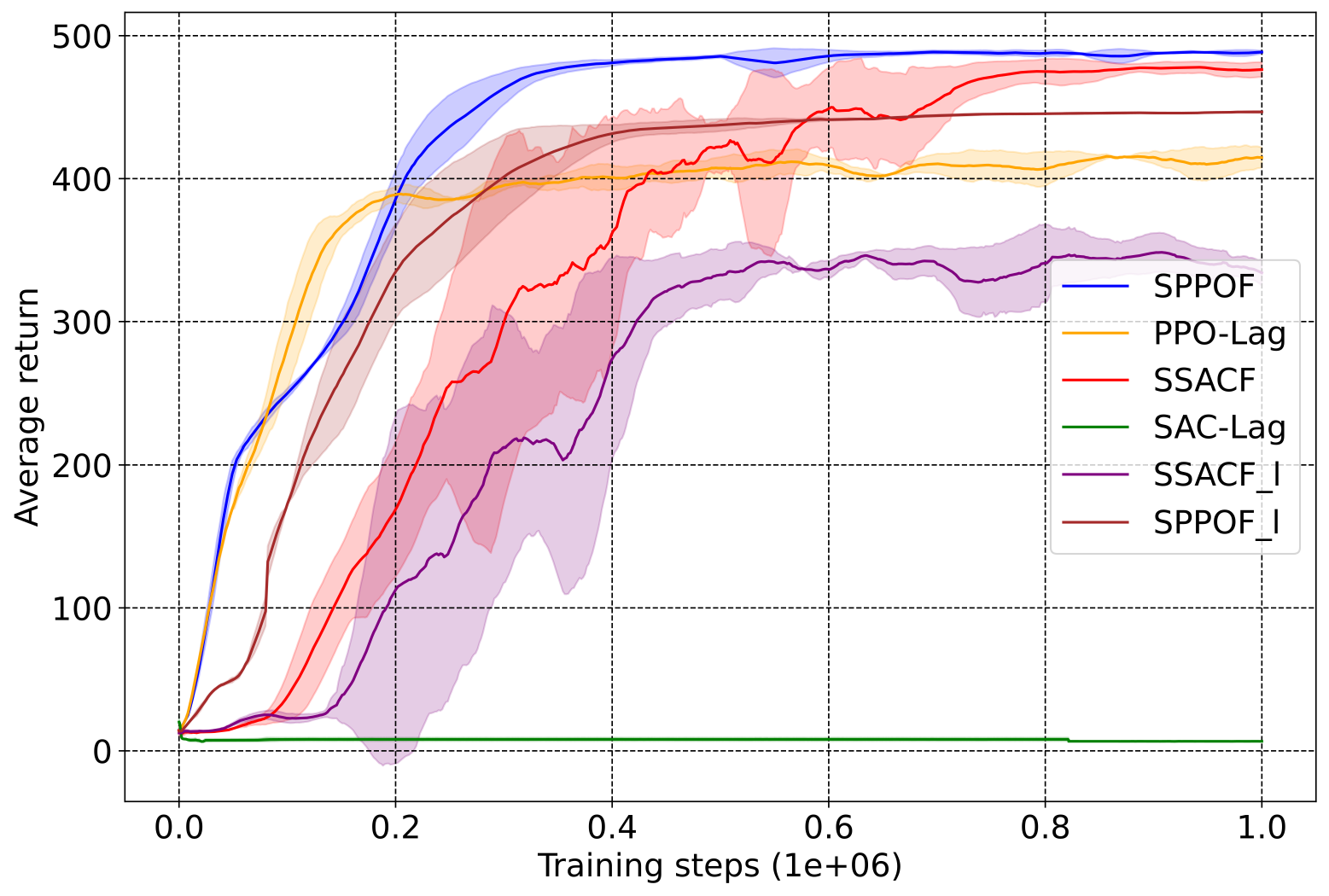}
\caption{The average return training curves of SPPOF, SSACF, PPO-Lag, SAC-Lag, SPPOF\_l and SSACF\_l by running 5 times with different seed. The lines and shaded areas represent the average return and the 95\% confidence interval, respectively.} \label{fig2}
\end{figure}

\begin{figure}[!h]
	\centering
	\includegraphics[width=0.55\textwidth]{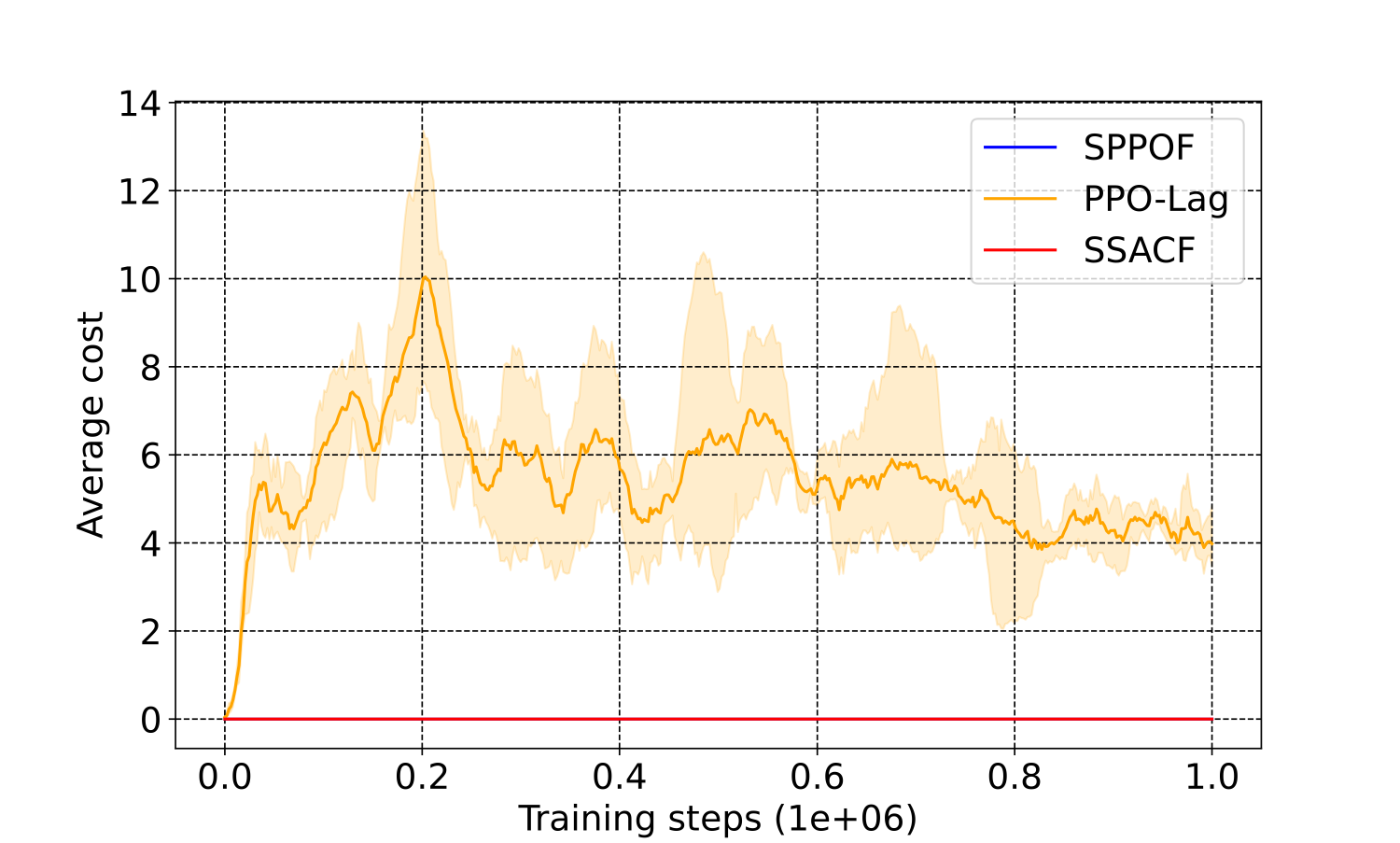}
	\caption{The average cost training curves of SPPOF\_l , PPO-Lag and SSACF\_l by running 5 times with different seed. The cost value is the number of collisions. The lines and shaded areas represent the average cost and the 95\% confidence interval, respectively.} \label{fig3}
\end{figure}

\begin{figure}[!h]
\centering
\includegraphics[scale=0.45,trim={6.5cm 3cm 0 4cm},clip]{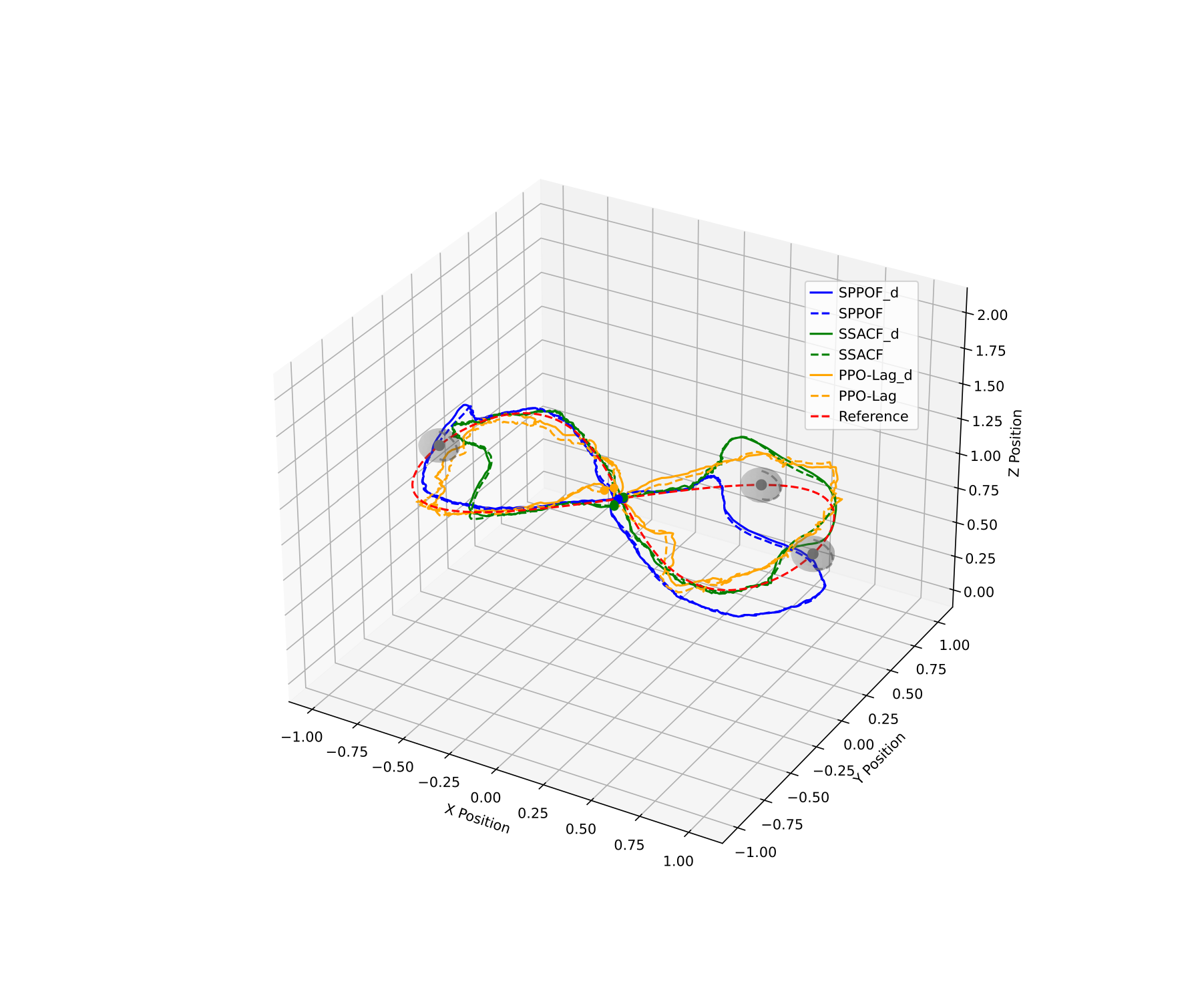}
\caption{Collision-free figure-8 tracking of SPPOF, SSACF and PPO-Lag, the solid line represents the trajectory without disturbances, the dashed line represents the trajectory under disturbances, and the gray spheres are obstacles.} \label{fig4}
\end{figure}

\begin{figure}[!h]
	\centering
	\includegraphics[scale=0.45,trim={6.5cm 3cm 0 4cm},clip]{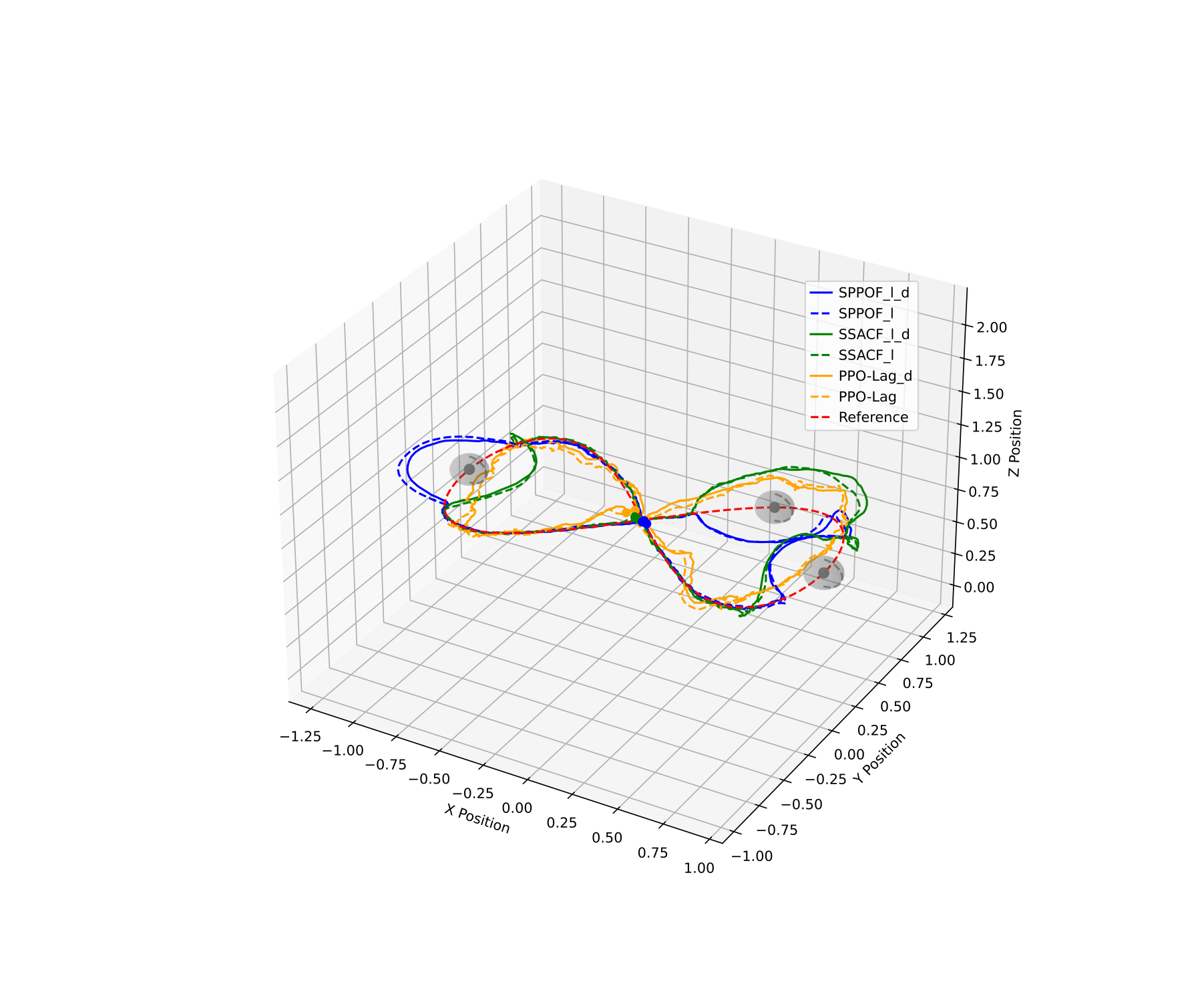}
	\caption{Collision-free figure-8 tracking of SPPOF\_l, SSACF\_l and PPO-Lag, the solid line represents the trajectory without disturbances, the dashed line represents the trajectory under disturbances, and the gray spheres are obstacles.} \label{fig5}
\end{figure}

\begin{figure}[!h]
	\centering
	\includegraphics[width=0.5\textwidth]{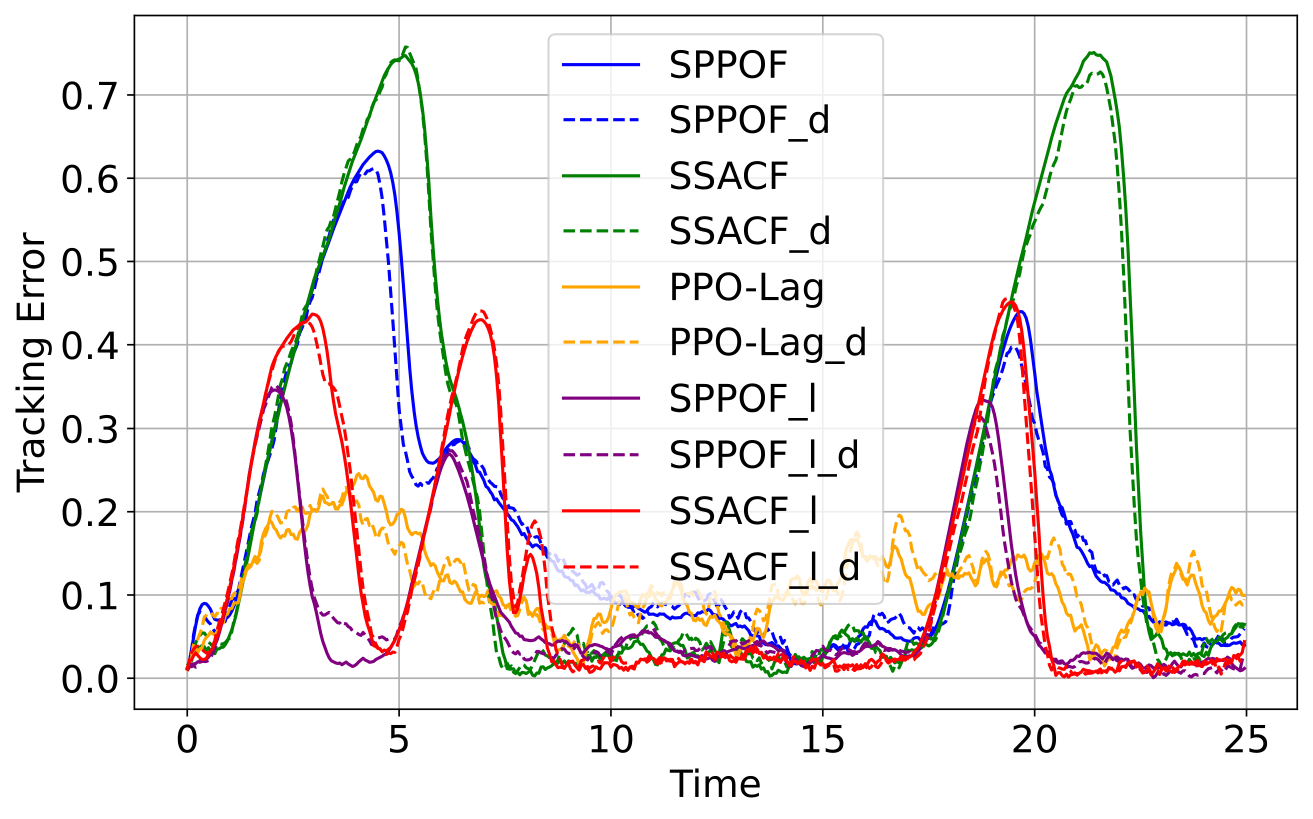}
	\caption{The tracking error of SPPOF, SSACF, PPO-Lag, SPPOF\_l and SSACF\_l. The peak indicates that Crazyflie is avoiding obstacle} \label{fig8}
\end{figure}

\begin{figure}[!h]
	\centering
	\includegraphics[width=0.5\textwidth]{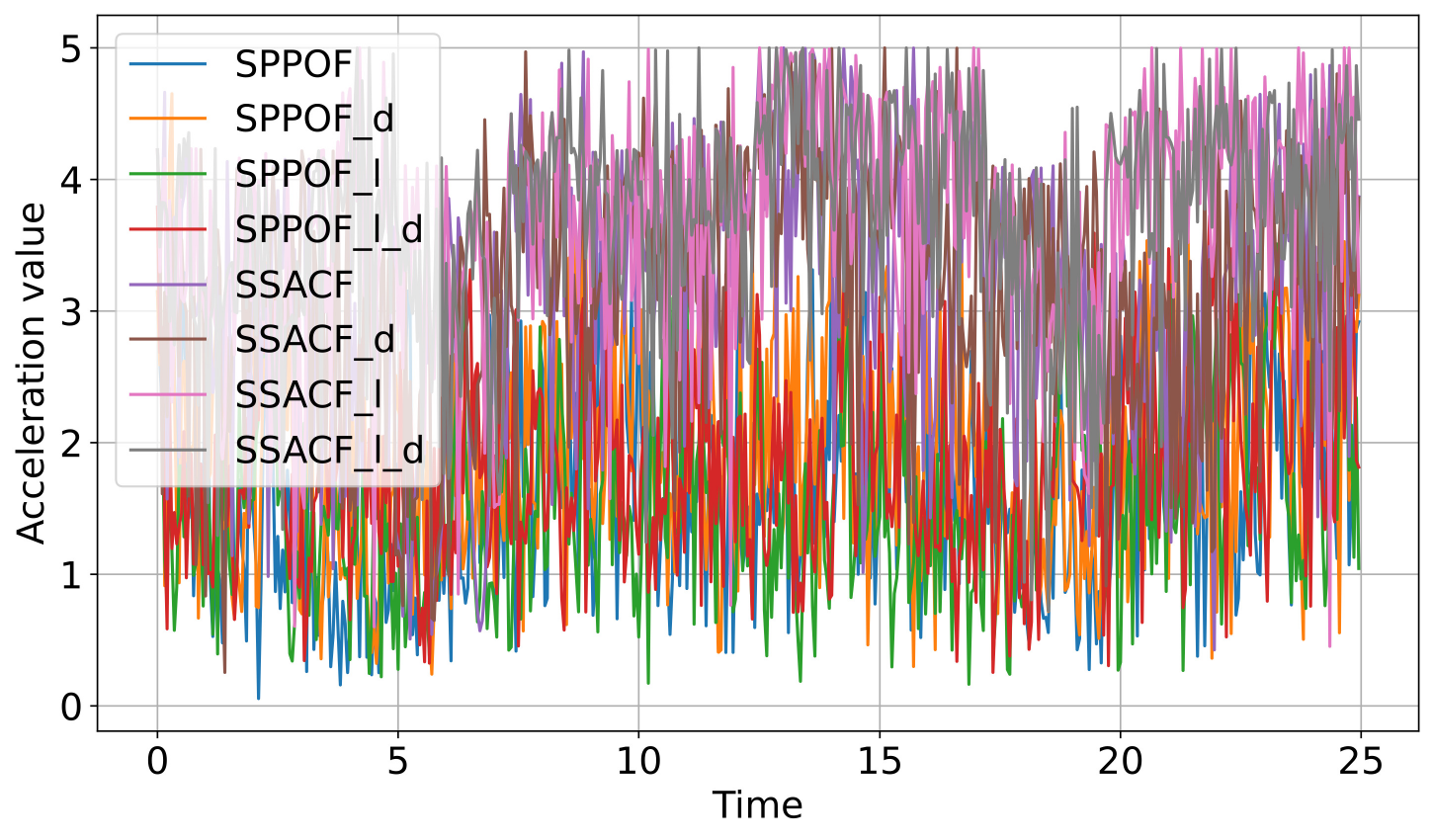}
	\caption{The acceleration $\mathbf{v}_{c}$ of SPPOF, SSACF, PPO-Lag, SPPOF\_l and SSACF\_l with input saturation $\|\mathbf{v}_{c}\|\leq 5\text{m}/\text{s}^{2}$.} \label{fig13}
\end{figure}

\begin{figure}[!h]
\centering
\includegraphics[width=0.5\textwidth]{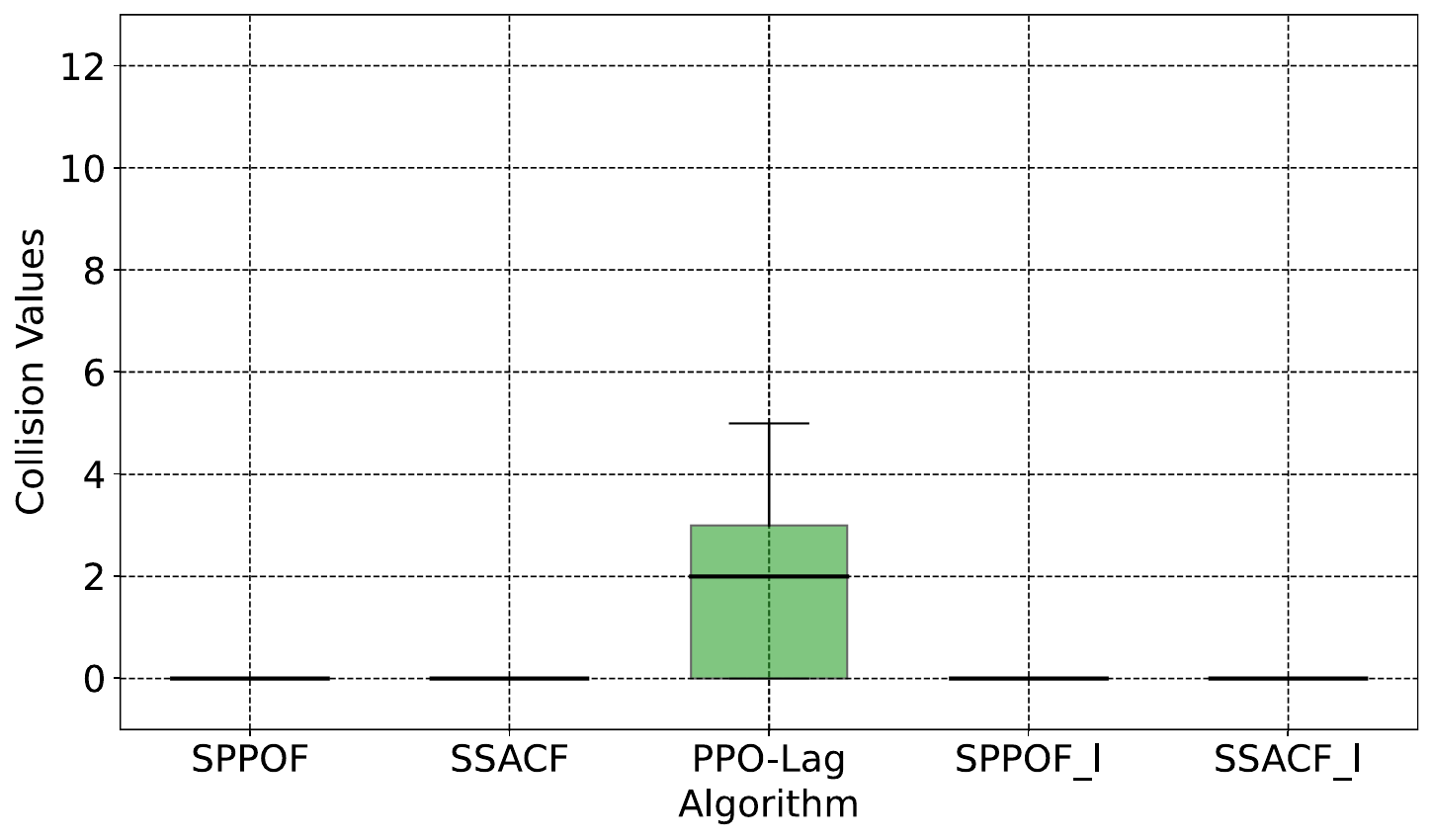}
\caption{The collision count of SPPOF, SSACF, PPO-Lag, SPPOF\_l and SSACF\_l. The red dot is outliers.} \label{fig6}
\end{figure}

\begin{figure}[!h]
	\centering
	\includegraphics[width=0.5\textwidth]{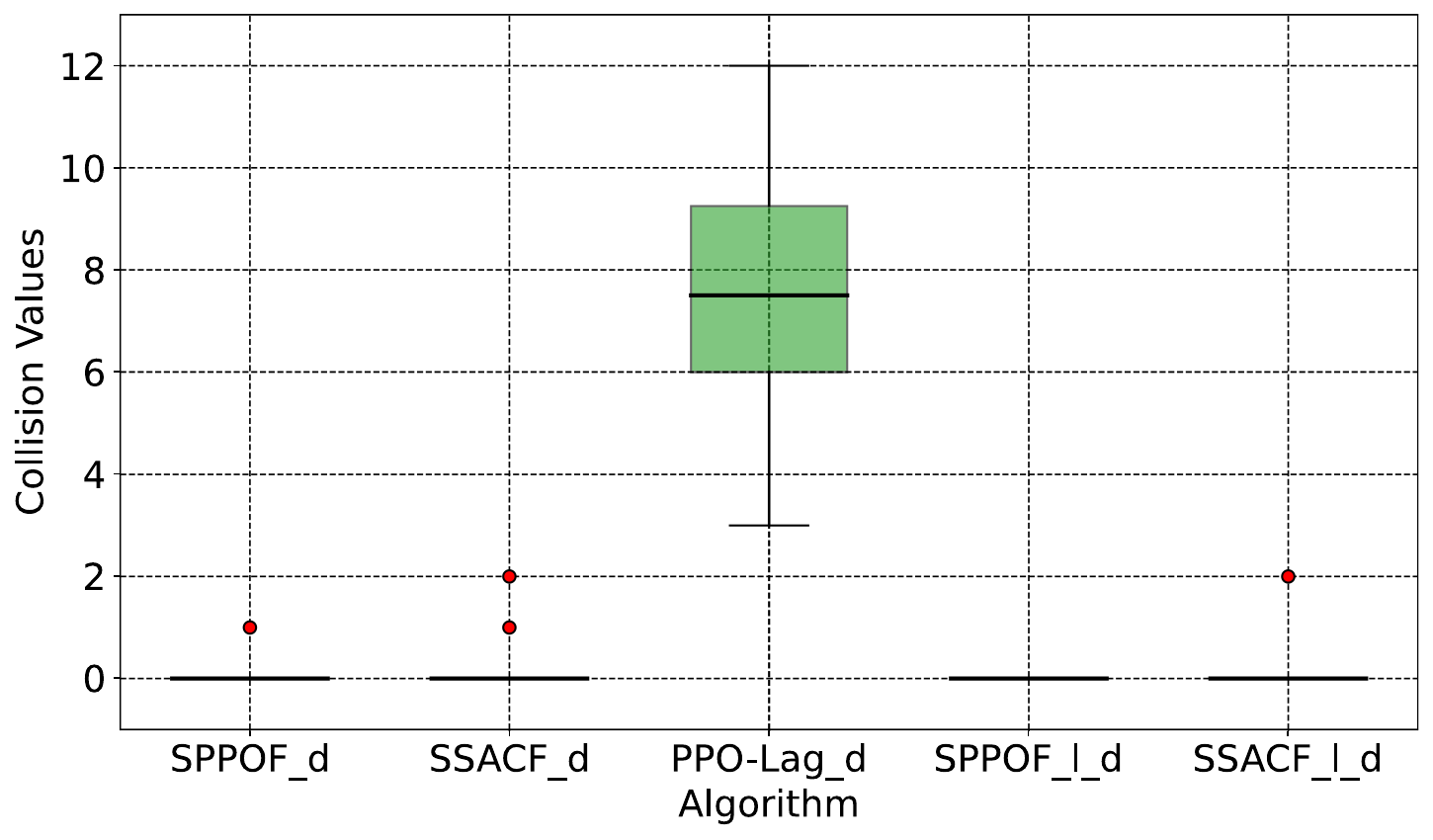}
	\caption{The collision count of SPPOF, SSACF, PPO-Lag, SPPOF\_l and SSACF\_l. $\_\boldsymbol{\mu}$ means under disturbances, the red dot is outliers.} \label{fig7}
\end{figure}

\begin{figure}[!h]
	\centering
	\includegraphics[scale=0.22,trim={2cm 0 1cm 0},clip]{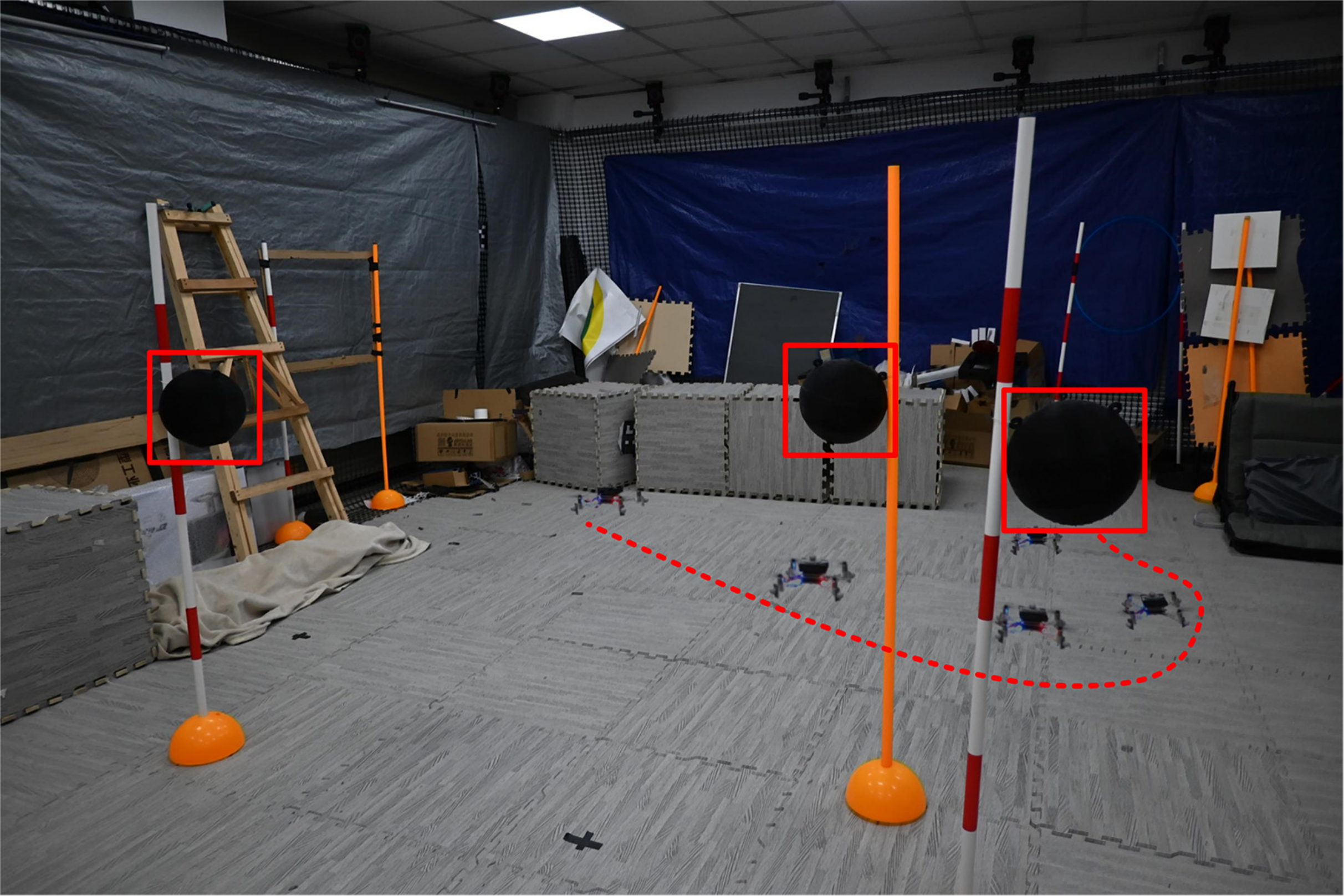}
	\caption{The figure-8 trajectory tracking of SPPOF, the red dashed line is the trajectory of Crazyflie 2.1, the black spheres in red square are obstacle.} \label{fig9}
\end{figure}

\begin{figure}[!h]
	\centering
	\includegraphics[scale=0.22,trim={2cm 0 1cm 0},clip]{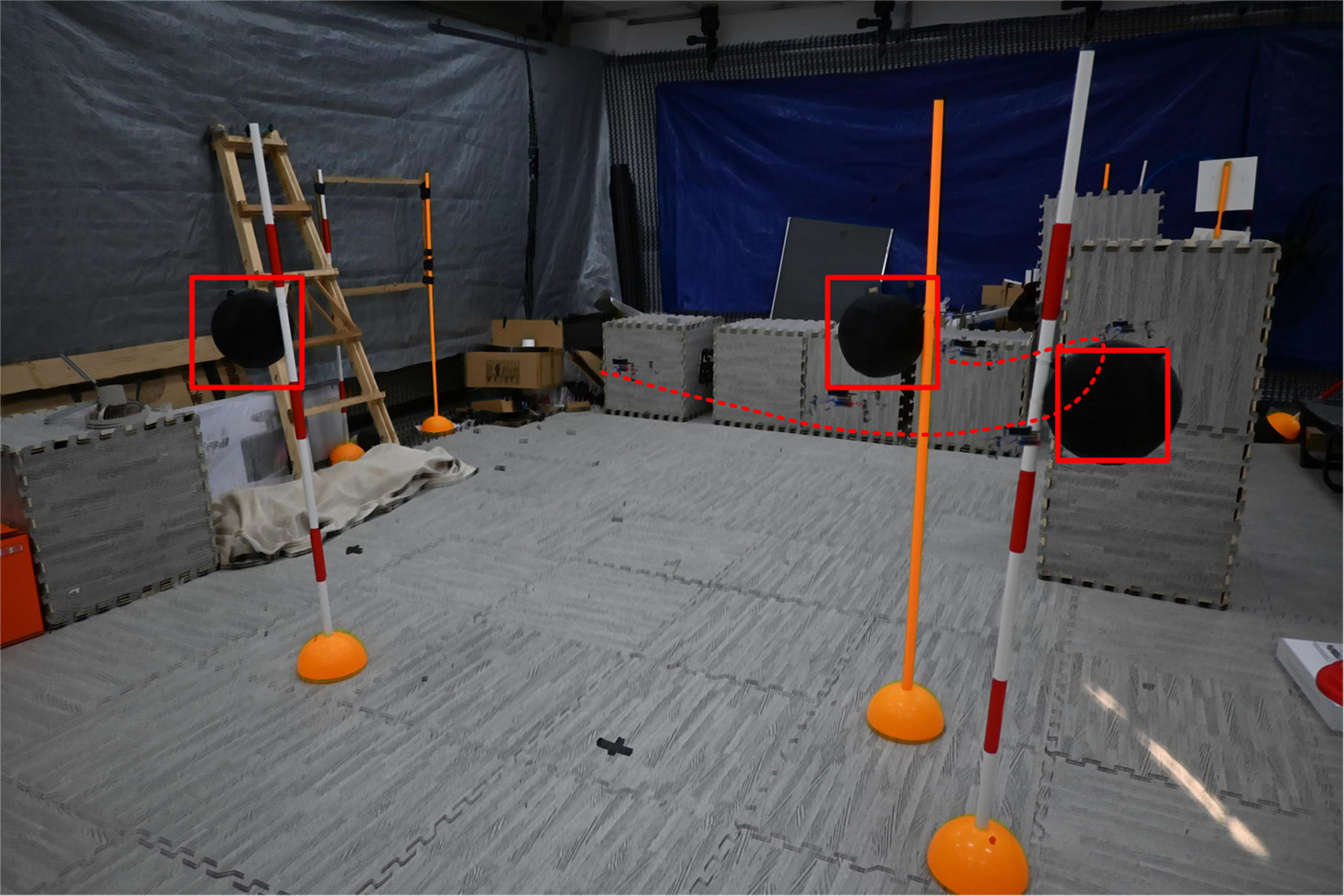}
	\caption{The figure-8 tracking of SPPOF\_l, the red dashed line is the trajectory of Crazyflie 2.1, the black spheres in red square are obstacle.} \label{fig10}
\end{figure}

\begin{figure}[!h]
	\centering
	\includegraphics[scale=0.22,trim={2cm 0 1cm 0},clip]{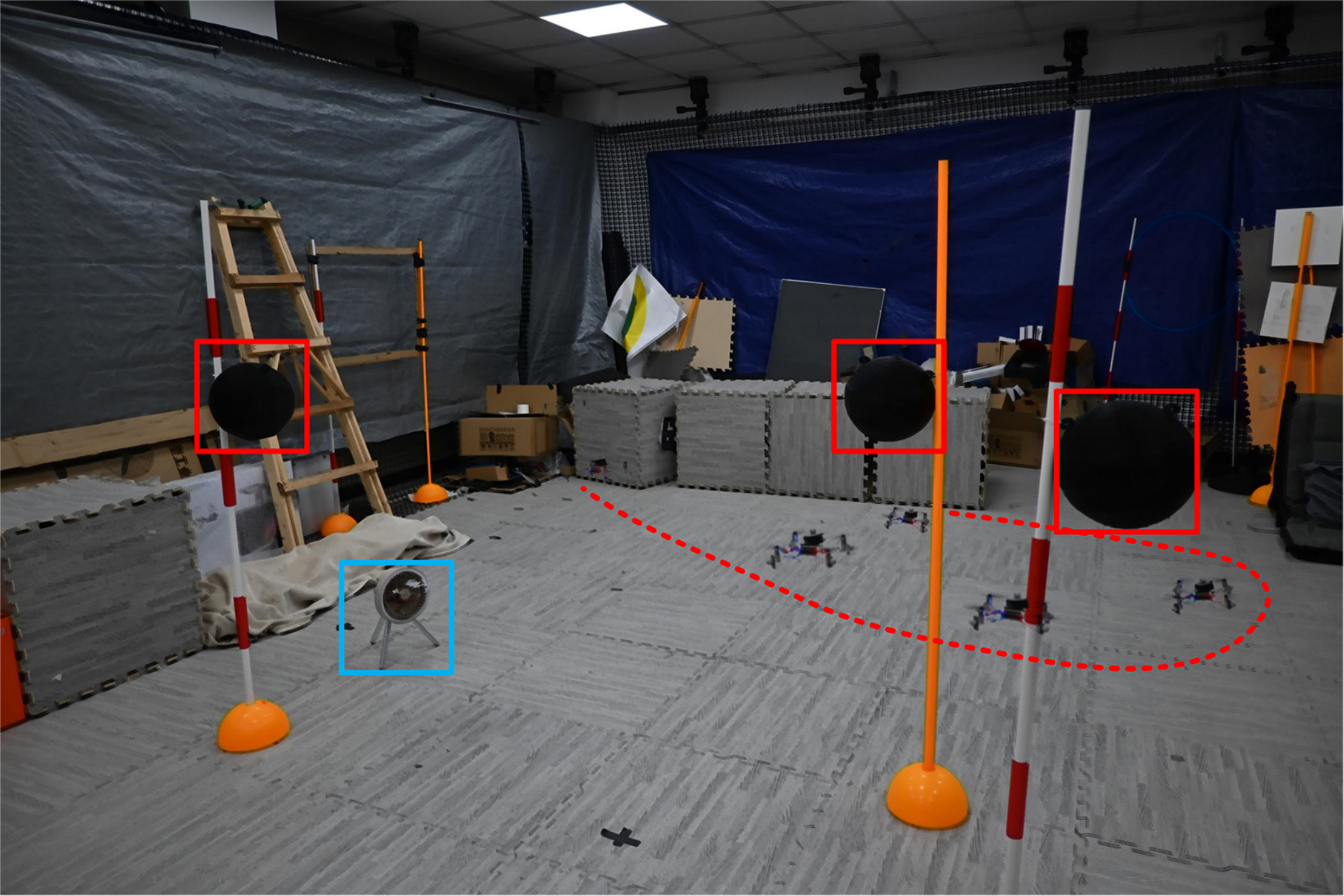}
	\caption{The figure-8 tracking of SPPOF, the red dashed line is the trajectory of Crazyflie under wind disturbances,  the blue square is a fan to generate wind, the black spheres in red square are obstacle.} \label{fig11}
\end{figure}

\begin{figure}[!h]
	\centering
	\includegraphics[scale=0.22,trim={2cm 0 1cm 0},clip]{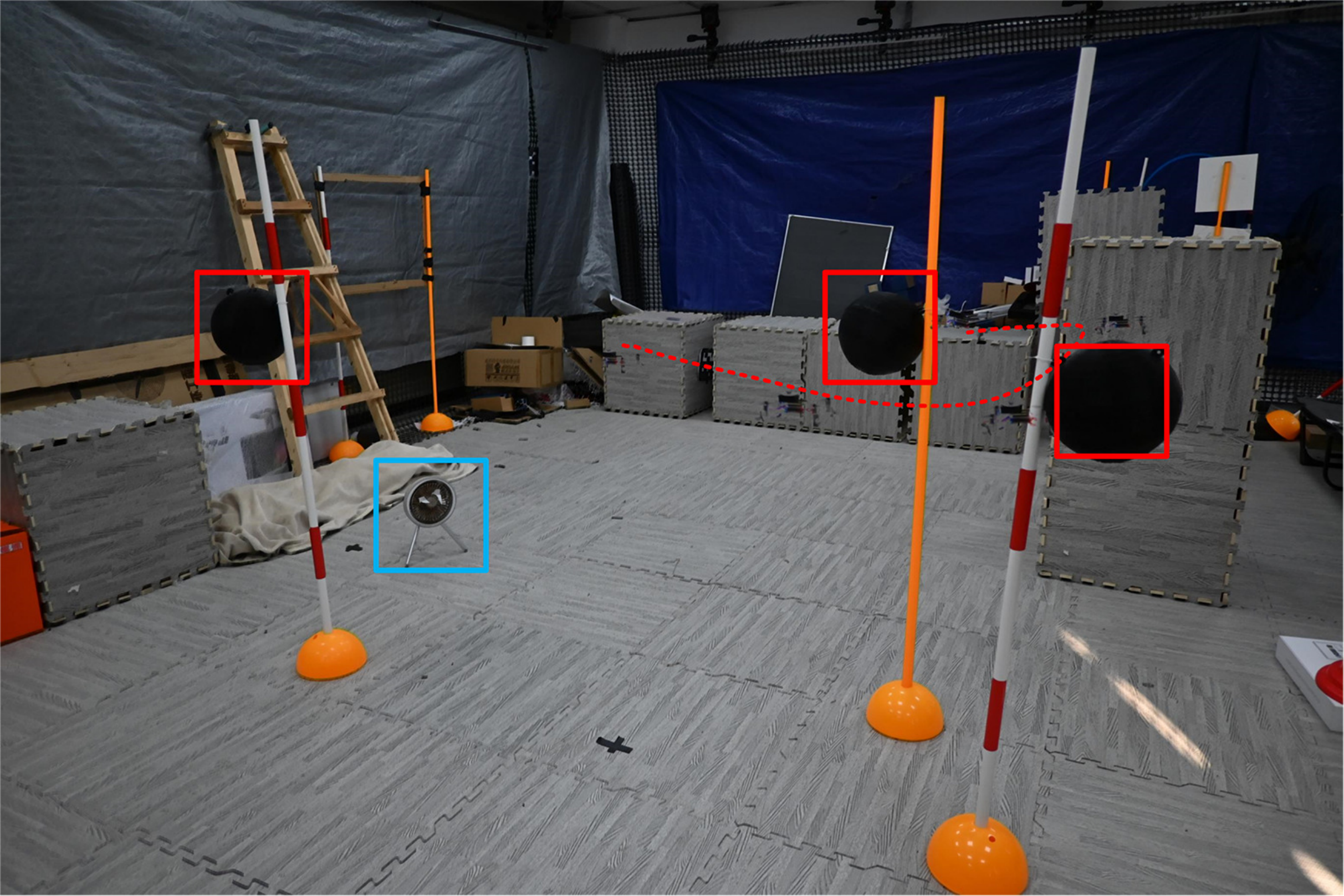}
	\caption{The figure-8 tracking of SPPOF\_l, the red dashed line is the trajectory of Crazyflie under wind disturbances,  the blue square is a fan to generate wind, the black spheres in red square are obstacle.} \label{fig12}
\end{figure}

\subsection{Tracking of a Figure-8 Trajectory}
In the model-free RL setting, the RL control action is required to control a multicopter to track a figure-8 trajectory without any obstacles and disturbances. After training, the safety filter is utilized to achieve collision avoidance with disturbances and input-saturation. The RL episode length is $500$ steps and time interval between each step is $0.05\text{s}$.
The observation is a  14-dimensional vector, which is composed of the current position and velocity  $\mathbf{x}=[\mathbf{p}^{T}, \mathbf{v}^{T}]$, the yaw angle $a_{\psi}$, the reference position $\mathbf{x}_{r}=[\mathbf{p}_{r}^{T}, \mathbf{v}_{r}^{T}]$, the reference yaw angle $a_{\psi,r}$. The RL control action is designed as a 4-dimensional vector $[\mathbf{v}_{c}^{T}, w_{c}^{T}]$ with $\|\mathbf{v}_{c}\|\leq5\text{m}/\text{s}^{2}$ and  $\|w_{c}\|\leq\frac{\pi}{3}\text{rad}/s$, it should be noted that this control action is complete according to \cite{mellinger2011minimum}. The reference trajectory is designed as $\mathbf{p}_{r}=[\sin(\frac{2\pi}{25}t)\text{m}, \frac{1}{2}\sin(\frac{4\pi}{25}t)\text{m}, 1\text{m}]$,  and $\dot{\mathbf{p}}_{r}$, $\dot{\mathbf{v}}_{r}$ can be derived from $\mathbf{p}_{r}$.  $\dot{a}_{\psi,r}=\text{atan2}\left(\frac{2\pi}{50} \cos\left(\frac{4\pi}{25} t\right), \frac{2\pi}{25} \cos\left(\frac{2\pi}{25} t\right)\right)$ is facing the direction of movement, $\text{atan2}$ denotes the inverse tangent function. Three sphere obstacles are placed and their specific positions of center are $\mathbf{p}_{1}=[1\text{m}, 0\text{m}, 1\text{m}]^{T}$, $\mathbf{p}_{2}=[0.5\text{m}, \frac{\sqrt{3}}{4}\text{m}, 1\text{m}]^{T}$ and $\mathbf{p}_{3}=[-1\text{m}, 0\text{m}, 1\text{m}]^{T}$ with radius $0.1\text{m}$. The safety distance $D_{s}$ is $0.15\text{m}$.   The disturbances $\boldsymbol{\mu}$ follows a uniform distribution $\text{Unif}[-1,1]$ in the simulation.

The reward function is constructed as
\begin{align*}
	r = \exp(-1.8(\|\mathbf{p}-\mathbf{p}_{r}\|^{2}+\|a_{\psi}-a_{\psi,r}\|^{2})),
\end{align*}
The nonlinear function $\exp(\cdot)$  provides a nonlinear mapping that better captures the relationship between tracking error and reward, such that small errors get relatively large rewards, and large errors get relatively small rewards. Additionally, it can prevent the reward from approaching zero when the error is large, maintaining numerical stability.

We employ on-policy PPO and off-policy SAC from the Stable Baselines3 as model-free RL algorithms, which are among the state-of-art algorithms currently. By introducing PPO and SAC in SRLF with two training methods, safe post filtering and safe learning filter, they are called SPPOF, SSACF and SPPOF$\_$l, SSACF$\_$l. The RL control action, RL critic networks are both modeled as 2-layer perceptrons with 128 hidden units. The activation function used in each unit is ReLU, and the final outputs of all networks are linear. The control output of filter $\mathbf{u}_{s}$ can be obtained by solving the QP with optimization tool CVXOPT with RCBF gain with disturbances and input saturation

Fig. \ref{fig2} shows the average return training curves for the algorithms, where PPO-Lag \cite{zhang2022conservative} and SAC-Lag \cite{ha2020learning} are safety learning methods. It can be obtained form Fig. \ref{fig2} that the training of SAC-Lag  failed since safety learning methods have to achieve tracking and collision avoidance at the same time, which increases the training burden and can even lead to training failures, as seen with SAC-Lag. In Fig. \ref{fig2}, the average return of PPO-Lag is less than SPPOF and SSACF because SPPOF and SSACF needn't consider collision avoidance. While considering  collision avoidance in training in Fig. \ref{fig3}, the  SPPOF\_l and SSACF\_l keep zero cost which implies no collision in training, and PPO-Lag  have multiple collisions during training.

We compared the figure-8 tracking trajectory of two safety filter methods with safety learning method PPO-Lag, which are shown in Fig. \ref{fig4} and Fig. \ref{fig5}. It can be obtained that safe post filtering and safe learning filtering both achieve collision-free tracking from Fig. \ref{fig4} and Fig. \ref{fig5}. The trajectory of the safe learning filtering method is closer to reference figure-8 trajectory since we use safety filter in training stage. The tracking error approaches zero when the Crazyflie 2.1 is far away from the obstacles, which is illustrated in Fig. \ref{fig8}.
The tracking acceleration $\mathbf{v}_{c}$ is shown in Fig. \ref{fig13}, which satisfied input saturation constraint $\|\mathbf{v}_{c}\|\leq5\text{m}/\text{s}^{2}$.

In order to verify the anti-disturbances capability of RCBF, we test 20 tests in our simulation, the boxplots are in Fig. \ref{fig6} and Fig. \ref{fig7}. The safety learning method PPO-Lag suffers an average of 7 collisions under disturbances, even without disturbances,  PPO-Lag would also take an average of 2 hits, while SRLF algorithms achieve zero collisions with the help of RCBF when outliers are excluded.

\subsection{Real-world Deployment}
In the real-world experiment,  a specific Crazyflie 2.1 is utilized with the parameters in \tabref{tab1}.
\begin{table}[htbp]
	\centering
	\caption{Crazyflie 2.1 Quadcopter Parameters} \label{tab1}
	\begin{tabular}{>{\centering\arraybackslash}m{0.2\textwidth} >{\centering\arraybackslash}m{0.2\textwidth}}
		\toprule
		Parameters &  Values  \\
		\hline
		Mass  & 28 g  \\
		Arm  &  3.97 cm \\
		Propeller radius & 2.31 cm \\
		Max speed & 30 km/h \\
		Thrust2weight & 1.88 \\
		\bottomrule
	\end{tabular}
\end{table}

The position $\mathbf{p}$, velocity $\mathbf{v}$ and yaw angle $\psi$ are obtained based on NOKOV Motion Capture System.

As shown in Fig. \ref{fig9} and Fig. \ref{fig10}, the Crazyflie 2.1 successfully achieves collision-free tracking in 3D space with two filter methods. In Fig. \ref{fig11} and  Fig. \ref{fig12}, we placed a fan to introduce external disturbances and the Crazyflie 2.1 still successfully implement tracking, which verifies the effective of our RCBF.

In summary, according to Fig. \ref{fig2}--\ref{fig12},  the proposed SRLF achieves collision-free trajectory tracking  under disturbances and input saturation in both  simulation and real-world.

\section{Conclusions}
In this paper, a new SRLF framework is proposed to realize multicopter collision-free trajectory tracking under input disturbances and saturation. By integrating a novel RCBF with its analysis techniques, the SRLF transforms potentially unsafe RL control inputs into safe commands. The RCBF gain is designed in control action to guarantee the safe set forward invariance under input disturbance. The safety filter allows RL training to proceed without explicitly considering safety constraints, simplifying the overall training process. The approach leverages QP to incorporate forward invariance of RCBF and input saturation constraints, ensuring rigorous guaranteed collision-free. Both simulation and real-world multicopter experiments demonstrate the effectiveness and excellent performance of SRLF in achieving collision-free tracking. This work contributes to the safer and more reliable deployment of multicopter in complex real-world applications.

\bibliographystyle{ieeetran}   
\bibliography{srlfilter}       

\end{document}